\documentclass{amsart} 

\usepackage{latexsym,amsfonts} 
\usepackage{amsmath,amssymb,graphics,setspace}
\usepackage{amsthm}
\usepackage{MnSymbol}
\usepackage{mathrsfs}

\usepackage{marvosym}
\usepackage{paralist}
\usepackage{color}

\usepackage{graphicx}
\usepackage{fontenc}%

\usepackage{hyperref}  
\hypersetup{linktocpage=true,colorlinks=true,linkcolor=blue,citecolor=blue,pdfstartview={XYZ 1000 1000 1}}

\newtheorem{theorem}{Theorem}

\newtheorem{corollary}{Corollary}

\theoremstyle{definition}

\newtheorem{example}{Example}

\newcommand{\abs}[1]{\lvert#1\rvert} 
\newcommand{\near}{\delta} 
\newcommand{\dnear}{\delta_{\Phi}} 
\newcommand{\dcap}{\mathop{\cap}\limits_{\Phi}} 
\newcommand{\dfar}{{\not\delta}_{\Phi}} 

\usepackage{pstricks,pst-text,pst-grad,pst-node,pst-tree,pst-plot,pst-3dplot,pst-barcode}
\usepackage{pstricks-add}

\usepackage{pst-solides3d}
\usepackage[tiling]{pst-fill}

\usepackage[verbose]{wrapfig}

\usepackage{subfigure}
\renewcommand{\thesubfigure}{\thefigure.\arabic{subfigure}}
\makeatletter
\renewcommand{\p@subfigure}{}
\renewcommand{\@thesubfigure}{\thesubfigure:\hskip\subfiglabelskip}
\makeatother

\begin{document} 

\title{Proximal Groupoid Patterns In Digital Images}

\author[E. A-iyeh]{E. A-iyeh$^{\beta}$}
\email{umaiyeh@myumanitoba.ca}
\address{\llap{$^{\beta}$\,} Computational Intelligence Laboratory,
University of Manitoba, WPG, MB, R3T 5V6, Canada}

\author[J.F. Peters]{J.F. Peters$^{\alpha}$}
\email{James.Peters3@umanitoba.ca, einan@adiyaman.edu.tr}
\address{\llap{$^{\alpha}$\,}Computational Intelligence Laboratory,
University of Manitoba, WPG, MB, R3T 5V6, Canada and
Department of Mathematics, Faculty of Arts and Sciences, Ad\.{i}yaman University, 02040 Ad\.{i}yaman, Turkey}

\thanks{The research has been supported by the Scientific and Technological Research
Council of Turkey (T\"{U}B\.{I}TAK) Scientific Human Resources Development
(BIDEB) under grant no: 2221-1059B211402463 and the Natural Sciences \&
Engineering Research Council of Canada (NSERC) discovery grant 185986.}


\date{}

\dedicatory{Dedicated to the Memory of Som Naimpally}

\begin{abstract}The focus of this article is on the detection and classification of patterns based on groupoids. The approach hinges on descriptive proximity of points in a set based on the neighborliness property. This approach lends support to image analysis and understanding and in studying nearness of image segments.  A practical application of the approach is in terms of the analysis of natural images for pattern identification and classification.
\end{abstract} 


\maketitle

\section{Introduction}
The focus of this paper is on groupoids in proximity spaces, which is an outgrowth of recent research~\cite{PIO2014gmn}.  A \emph{groupoid} is a system $S\left(\circ \right)$ consisting of a nonempty set $S$ together with a binary operation $\circ$ on $S$.  A binary operation $\circ:S\times S\rightarrow S$ maps $S\times S$ into $S$, where $S\times S$ is the set of all ordered pairs of elements of $S$.

Groupoids can either be spatially near in a traditional Lodato proximity space~\cite{Lodato1962,Lodato1964,Lodato1966} (see, also,~\cite{Naimpally2009,Naimpally1970,Peters2016isrl102compProximity}) or descriptively near in descriptive proximity space.
A descriptive proximity space~\cite{Peters2012ams,Peters2013mcsintro} is an extension of an Lodato proximity space.  This extension is made possible by the introduction of feature vectors that describe each point in a proximity space.  Sets $A,B$ in an EF-proximity space $X$ are near, provided $A\cap B\neq \emptyset$.  Sets $A,B$ in a descriptive proximity space $X$ are near, provided there is at least one pair of points $a\in A, b\in B$ with matching descriptions. 
 
In proximity spaces, groupoids provide a natural structuring of sets of points.  For a proximity space $X$, the structuring of $X$ results from a proximity relation on $X$.  Nonempty subsets $A$ of $X$ are further structured as a result of binary operations $\circ$ on $A$, yielding a groupoid denoted by $(A,\circ)$.    
These operations often lead to descriptions of the sets endowed for structures with potential applications in proximity spaces for pattern description, recognition and identification. In this work, the applications of descriptive proximity-based groupoids in digital image analysis is explored.  A practical application of the proposed approach is in terms of the analysis of natural images for pattern identification and classification.


\section{Preliminaries}
In a Kov\'{a}r discrete space, a non-abstract point has a location and features that can be measured~\cite[\S 3]{Kovar2011}. Let $X$ be a nonempty set of non-abstract points in a proximal relator space $(X,\mathcal{R_{\delta}})$ and let $\Phi = \left\{\phi_1,\dots,\phi_n\right\}$ a set of probe functions that represent features of each $x\in X$.    For example, leads to a proximal view of sets of picture points in digital images~\cite{Peters2013mcsintro}. A \emph{probe function} $\Phi:X\rightarrow \mathbb{R}$ represents a feature of a sample point in a picture.  Let $\Phi(x) = (\phi_1(x),\dots,\phi_n(x))$ denote a feature vector for $x$, which provides a description of each $x\in X$.  To obtain a descriptive proximity relation (denoted by $\delta_{\Phi}$), one first chooses a set of probe functions.   Let $A,B\in 2^X$ and $\mathcal{Q}(A),\mathcal{Q}(B)$ denote sets of descriptions of points in $A,B$, respectively.  For example, $\mathcal{Q}(A) = \left\{\Phi(a): a\in A\right\}$.
Let $X$ be a nonempty set. A \textit{Lodato proximity $\delta$} is a relation on $\mathscr{P}(X)$ which satisfies the following axioms for all subsets $A, B, C $ of $X$:\\
\begin{description}
\item[{\rm\bf (P0)}] $\emptyset \not\delta A, \forall A \subset X $.
\item[{\rm\bf (P1)}] $A\ \delta\ B \Leftrightarrow B \delta A$.
\item[{\rm\bf (P2)}] $A\ \cap\ B \neq \emptyset \Rightarrow A \near B$.
\item[{\rm\bf (P3)}] $A\ \delta\ (B \cup C) \Leftrightarrow A\ \delta\ B $ or $A\ \delta\ C$.
\item[{\rm\bf (P4)}] $A\ \delta\ B$ and $\{b\}\ \delta\ C$ for each $b \in B \ \Rightarrow A\ \delta\ C$. \qquad \textcolor{blue}{$\blacksquare$}
\end{description}
\vspace{3mm}
\emph{Descriptive Lodato proximity} was introduced in~\cite{DBLP:series/isrl/2014-63}.  Let $X$ be a nonempty set, $A,B\subset X, x\in X$.  Recall that $\Phi(x)$ is a feature vector that describes $x$, $\Phi(A)$ is the set of feature vectors that describe points in $A$.  The descriptive intersection between $A$ and $B$ (denoted by $A\ \dcap\ B$) is defined by
\[
A\ \dcap\ B = \left\{x\in A\cup B: \Phi(x)\in\Phi(A)\ \mbox{and}\ \Phi(x)\in\Phi(B)\right\}.
\]
A \textit{Descriptive Lodato proximity $\dnear$} is a relation on $\mathscr{P}(X)$ which satisfies the following axioms for all subsets $A, B, C $ of $X$, which is a descriptive Lodato proximity space:\\
\begin{description}
\item[{\rm\bf (dP0)}] $\emptyset\ \dfar\ A, \forall A \subset X $.
\item[{\rm\bf (dP1)}] $A\ \dnear\ B \Leftrightarrow B\ \dnear\ A$.
\item[{\rm\bf (dP2)}] $A\ \dcap\ B \neq \emptyset \Rightarrow\ A\ \dnear\ B$.
\item[{\rm\bf (dP3)}] $A\ \dnear\ (B \cup C) \Leftrightarrow A\ \dnear\ B $ or $A\ \dnear\ C$.
\item[{\rm\bf (dP4)}] $A\ \dnear\ B$ and $\{b\}\ \dnear\ C$ for each $b \in B \ \Rightarrow A\ \dnear\ C$. \qquad \textcolor{blue}{$\blacksquare$}
\end{description}
\vspace{3mm}
The expression $A\  \delta_{\Phi}\ B$ reads $A\ \mbox{\emph{is descriptively near}}\ B$.  Similarly, $A\  \underline{\delta}_{\Phi}\ B$ reads $A\ \mbox{\emph{is descriptively far from}}\ B$.  The descriptive proximity of $A$ and $B$ is defined by
\[
A\ \dcap\ B \neq \emptyset \Rightarrow\ A\ \dnear\ B.
\] 
That is, nonempty descriptive intersection of $A$ and $B$ implies the descriptive proximity of $A$ and $B$.  Similarly, the Lodato proximity of $A,B\subset X$ (denoted $A\ \delta\ B$) is defined by
\[
A\ \cap\ B \neq \emptyset \Rightarrow A\ \near\ B.
\] 
In an ordinary metric closure space~\cite[\S 14A.1]{Cech1966} $X$, the closure of $A\subset X$ (denoted by $\mbox{cl}(A)$) is defined by
\begin{align*}
\mbox{cl}(A) &= \left\{x\in X: D(x,A)=0\right\},\ \mbox{where}\\
D(x,A) &= inf\left\{d(x,a): a\in A\right\},
\end{align*}
{\em i.e.}, $\mbox{cl}(A)$ is the set of all points $x$ in $X$ that are close to $A$ ($D(x,A)$ is the Hausdorff distance \cite[\S 22, p. 128]{Hausdorff1914} between $x$ and the set $A$ and $d(x,a)=\left|x-a\right|$ (standard distance)).  Subsets $A,B\in 2^X$ are spatially near (denoted by $A\ \delta\ B$), provided the intersection of closure of $A$ and the closure of $B$ is nonempty, i.e., $\mbox{cl}(A)\cap\mbox{cl}(B)\neq \emptyset$.  That is, nonempty sets are spatially near, provided the sets have at least one point in common. 

\noindent The Lodato proximity relation $\delta$ (called a \emph{discrete} proximity) is defined by 
\[
\delta = \left\{(A,B)\in 2^X\times 2^X: \mbox{cl}(A)\ \cap\ \mbox{cl}(B)\neq\emptyset\right\}.
\]
The pair $(X,\delta)$ is called an EF-proximity space.  In a proximity space $X$, the closure of $A$ in $X$ coincides with the intersection of all closed sets that contain $A$.  

\begin{theorem}\label{thm:cl}{\rm \cite{Smirnov1952}}
The closure of any set $A$ in the proximity space $X$ is the set of points $x\in X$ that are close to $A$.
\end{theorem}

\begin{corollary}
The closure of any set $A$ in the proximity space $X$ is the set of points $x\in X$ such that $x\ \delta\ A$.
\end{corollary}
\begin{proof}
From Theorem~\ref{thm:cl}, $\mbox{cl}(A) = \left\{x\in X: \left\{x\right\}\ \delta\ A\right\}$.
\end{proof}

The expression $A\  \delta_{\Phi}\ B$ reads $A\ \mbox{\emph{is descriptively near}}\ B$. The relation $\delta_{\Phi}$ is called a \emph{descriptive proximity relation}.  Similarly, $A\  \underline{\delta}_{\Phi}\ B$ denotes that $A$ is descriptively far (remote) from $B$.  The descriptive proximity of $A$ and $B$ is defined by
\[
A\ \delta_{\Phi}\ B \Leftrightarrow \mathcal{Q}(\mbox{cl}(A)) \cap \mathcal{Q}(\mbox{cl}(B)) \neq \emptyset.
\] 
\noindent 

\noindent The \emph{descriptive intersection} $\mathop{\cap}\limits_{\Phi}$ of $A$ and $B$ is defined by
\[
A\ \mathop{\cap}\limits_{\Phi}\ B = \left\{x\in A\cup B:\Phi(x)\in \mathcal{Q}(\mbox{cl}(A))\ \mbox{and}\ \Phi(x)\in \mathcal{Q}(\mbox{cl}(B))\right\}.
\]
That is, $x\in A\cup B$ is in $\mbox{cl}(A)\ \mathop{\cap}\limits_{\Phi}\ \mbox{cl}(B)$, provided $\Phi(x) = \Phi(a) = \Phi(b)$ for some $a\in \mbox{cl}(A), b\in \mbox{cl}(B)$.

\noindent 
The descriptive proximity relation $\delta_{\Phi}$ is defined by
\[
\delta_{\Phi} = \left\{(A,B)\in 2^X\times 2^X:
                       \mbox{cl}(A)\ \mathop{\cap}\limits_{\Phi}\ \mbox{cl}(B)\neq\emptyset\right\}.
\]

The pair $(X,\delta_{\Phi})$ is called a descriptive Lodato proximity space.  In a descriptive proximity space $X$, the descriptive closure of $A$ in $X$ contains all points in $X$ that are descriptively close to the closure of $A$.  Let $\delta_{\Phi}(A,B) = 0$ indicate that $A$ is descriptively close to $B$. The \emph{ descriptive closure of a set} $A$ (denoted by $\mbox{cl}_{\Phi}(A)$) is defined by
\[
\mbox{cl}_{\Phi}(A) = \left\{x\in X: \Phi(x) \in \mathcal{Q}(\mbox{cl}(A))\right\}.
\]
That is, $x\in X$ is in the descriptive closure of $A$, provided $\Phi(x)$ (description of $x$) matches $\Phi(a)\in \mathcal{Q}(\mbox{cl}(A))$ for at least one $a\in \mbox{cl}(A)$.

\begin{theorem}\label{thm:dcl}{\rm \cite{DBLP:series/isrl/2014-63}}
The descriptive closure of any set $A$ in the descriptive proximity space $X$ is the set of points $x\in X$ that are descriptively close to $A$.
\end{theorem}

\begin{corollary}
The description closure of any set $A$ in the descriptive proximity space $X$ is the set of points $x\in X$ such that $x\ \delta_{\Phi}\ A$.
\end{corollary}
\begin{proof}
From Theorem~\ref{thm:dcl}, $\mbox{cl}_{\Phi}(A) = \left\{x\in X: \left\{x\right\}\ \delta_{\Phi}\ A\right\}$.
\end{proof}

\section{Descriptive Proximal Groupoids}
In the descriptively near case, we can consider either partial groupoids
within the same set or partial groupoids that belong to disjoint sets. That
is, the descriptive nearness of groupoids is not limited to partial
groupoids that have elements in common. Implicitly, this suggests a new form
of the connectionist paradigm, where sets are not spatially connected (sets
with no elements in common) but are descriptively connected (sets have
elements belong to the descriptive intersection of the sets).  
For more about this new form of connectedness, see, {\em e.g.},~\cite{PetersGuadagni2016connectedness}.

Recall that a partial binary operation on a set $X$ is a mapping of a nonempty
subset of $X\times X$ into $X$. A partial groupoid is a system $S\left(
\ast \right) $ consisting of a nonempty set $S$ together with a partial
binary operation ``$\ast$'' on $S$. Let
\begin{equation*}
\mathcal{Q}\left( A\right) =\left\{ \Phi \left( a\right) :a\in A\right\},
\text{ set of descriptions of members of }A.
\end{equation*}
Next, to arrive at a partial descriptive groupoid, let ``$\ast _{\Phi }$'' be a
partial descriptive binary operation defined by
\begin{equation*}
\ast _{\Phi }:\mathcal{Q}(S)\times \mathcal{Q}(S)\longrightarrow \mathcal{Q}(S).
\end{equation*}

\section{Neighbourliness in Proximal Groupoids}
In general, given a finite groupoid $A(\circ)$ represented by an undirected graph, for $x,y\in A$  let $x\sim y$ denote the fact that $x,y$ are neighbours.  In particular, given a finite descriptive proximal groupoid $A(\circ_{\Phi})$, for $\Phi(x),\Phi(y)\in \mathcal{Q}(A)$,  let $\Phi(x)\sim \Phi(y)$ denote the fact that $\Phi(x),\Phi(y)$ are neighbours, {\em i.e.}, points $x,y\in A$ have matching descriptions.

\begin{example}\rm
Let $(X,\delta_{\Phi})$ be a descriptive proximity space.   Then let $\left\{\Phi(x)\right\}$ denote an equivalence class represented by the feature vector $\Phi(x)$ that describes $x\in X$.  Then $\Phi(y)\in \left\{\Phi(x)\right\}$, provided $\Phi(y) = \Phi(x)$, {\em i.e.}, $x,y\in X$ have matching descriptions.  Then, for example, define the descriptive groupoid $A(\circ_{\Phi})$ such that
\[
A = \left\{\Phi(x)\in \mathcal{Q}(X): \Phi(x)\in  \left\{\Phi(x)\right\}\right\}.
\]
and
\[
\circ_{\Phi}(\Phi(x),\Phi(y)) = \Phi(x).
\]
Then $\Phi(x),\Phi(y)\in \mathcal{Q}(A)$ are neighbourly in groupoid $A$, provided $\Phi(x)$ and $\Phi(y)$ belong to the same equivalence class, {\em i.e.}, $\Phi(x)\in  \left\{\Phi(y)\right\}$. 
\qquad \textcolor{blue}{$\blacksquare$}
\end{example}

The notion of neighbourliness of elements in a groupoid extends to neighbourliness between disjoint groupoids.  Let $X,Y$ be disjoint descriptive proximity spaces and let $A(\circ_{\Phi}),B(\bullet_{\Phi})$ be descriptive proximal groupoids such that $A\subset X, B\subset Y$.  Then groupoids $A(\circ_{\Phi})$ and $B(\bullet_{\Phi})$ are neighbourly, provided $a\sim b$ for some $a\in A, b\in B$.  This leads to the following result.

\begin{theorem}\label{thm:nbdPattern}
Let $(X,\delta_{\Phi}),(Y,\delta_{\Phi})$ be descriptive proximity spaces, $A\subset X, B\subset Y$.  If groupoids $A(\circ_{\Phi}),B(\bullet_{\Phi})$ are neighbourly, then $A\ \delta_{\Phi}\ B$.
\end{theorem}
\begin{proof}
Immediate from the definition of descriptive near sets~\cite{Peters2013mcsintro}.
\end{proof}

\begin{figure}[!ht] 
\begin{center}
\begin{pspicture}
(0,0.5)(0,3.5)
  \includegraphics[width=80mm]{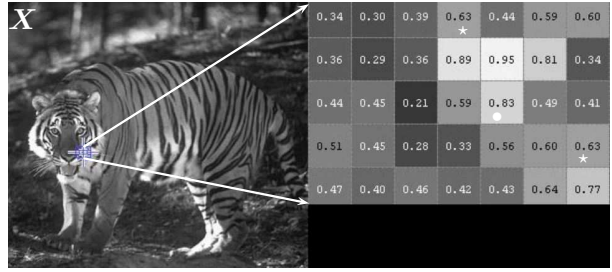}
	\psline[linecolor=white,linewidth=0.30mm]{->}(-7.0,1.6)(-4.0,3.5)
	\psline[linecolor=white,linewidth=0.30mm]{->}(-7.0,1.45)(-4.0,0.85)
	\rput(-7.8,3.3){\large \textcolor{white}{$\boldsymbol{X}$}}
	\rput(-1.95,3.15){\footnotesize \textcolor{white}{$\boldsymbol{\star}$}}
	\rput(-1.5,2.0){\footnotesize \textcolor{white}{$\boldsymbol{\bullet}$}}
	\rput(-0.35,1.45){\footnotesize \textcolor{white}{$\boldsymbol{\star}$}}
\end{pspicture}
\end{center}
\caption{Image $X$\ $\longrightarrow$\ $A(\circ)$}\label{fig:tiger1}
\end{figure}

\begin{figure}[!ht] 
\begin{center}
\begin{pspicture}
(0,0.5)(0,3.5)
  \includegraphics[width=98mm]{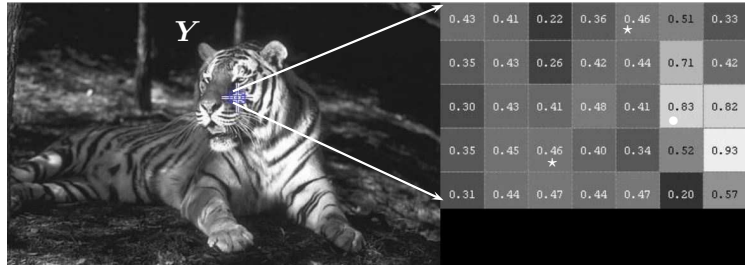}
	\psline[linecolor=white,linewidth=0.30mm]{->}(-6.8,2.3)(-4.0,3.45)
	\psline[linecolor=white,linewidth=0.30mm]{->}(-6.8,2.15)(-4.0,0.85)
	\rput(-7.4,3.1){\large \textcolor{white}{$\boldsymbol{Y}$}}
	\rput(-1.55,3.12){\footnotesize \textcolor{white}{$\boldsymbol{\star}$}}
	\rput(-0.94,1.91){\footnotesize \textcolor{white}{$\boldsymbol{\bullet}$}}
	\rput(-2.55,1.35){\footnotesize \textcolor{white}{$\boldsymbol{\star}$}}
\end{pspicture}
\end{center}
\caption{Image $Y$\ $\longrightarrow$\ $B(\circ)$}\label{fig:tiger2}
\end{figure}

\section{Application: Digital Proximal Groupoids}
A \emph{digital groupoid} is a groupoid in a digital image.
Let $X$ be a digital image and let $A$ be a nonempty subset in $X$ endowed with a binary operation $\circ$. In that case, $A(\circ)$ is a digital groupoid.  

\begin{example}\label{ex:op} Let $X$ be a greyscale digital image and let $A$ be a subimage in $X$.   Let $x,y$ be pixels in the submimage $A$ in Fig.~\ref{fig:tiger1} and let $\phi:X\rightarrow\mathbb{R}$ be defined by $\phi(x) =$ greyscale intensity of $x\in X$.  Further, define the set $\Phi = \left\{\phi\right\}$.  Hence, $\mathcal{Q}(A)$ is the set of feature vectors $\Phi(x)$, where each feature vector contains a single number $\phi(x)$, which is the intensity of $x\in A$.  Then define the binary operation $\circ:\mathcal{Q}(A)\times\mathcal{Q}(A)\rightarrow \mathcal{Q}(A)$ by
\[
\Phi(x) \circ \Phi(y) = min\left\{\Phi(x),\Phi(y)\right\}, \mbox{where}\ \Phi(x),\Phi(y)\ \mbox{are pixel intensities}.
\]
From this, we obtain the digital groupoid $\mathcal{Q}(A)(\circ)$.
\qquad \textcolor{blue}{$\blacksquare$}
\end{example}

\subsection{Neighbourliness of digital groupoid elements}

Let $X$ be a digital image shown in Fig.~\ref{fig:tiger1} and select the subset of pixel intensities $A\subset X$.   Let $\Phi$ be a set of probe functions that represent features of pixels in $X$ and let $\mathcal{Q}(A)$ be sets of feature vectors that describe $a\in A$. Let the binary operation $\circ$ to be that given in Example~\ref{ex:op}.  Hence, $A(\circ)$ is a digital groupoid.  Next, define the pseudometric $d:\mathcal{Q}(A)\times \mathcal{Q}(A)\rightarrow\mathbb{R}$ to be
\[
d(\Phi(x), \Phi(y)) = \abs{\Phi(x) - \Phi(y)}: \Phi(x),\Phi(y)\in \mathcal{Q}(A).
\]
Points $x,y\in A$ are neighbourly, provided $d(\Phi(x), \Phi(y)) = 0$.
\begin{example}\label{ex:neighbourly}
Let $X$ be the digital image shown in Fig.~\ref{fig:tiger1} and select the subset of pixel intensities $A\subset X$ (also shown in Fig.~\ref{fig:tiger1}).  Let $\Phi = \left\{\phi\right\}$, where $\phi(x)$ equals the greyscale intensity of $x\in A$.  Let $\Phi(x), \Phi(y)$ be the pixel intensities labelled with a white $\star$ in Fig.~\ref{fig:tiger1}. Observe that $\Phi(x) = \Phi(y) = 0.63$.  Hence, $\Phi(x), \Phi(y)$ are neighbourly elements of this digital groupoid.  There are other instances of neighbourly elements in the same groupoid.
\qquad \textcolor{blue}{$\blacksquare$}
\end{example}

\begin{figure}[!ht]
\begin{center}
 \subfigure[$X\longrightarrow A2(\circ)$]{
  \includegraphics[width=57 mm]{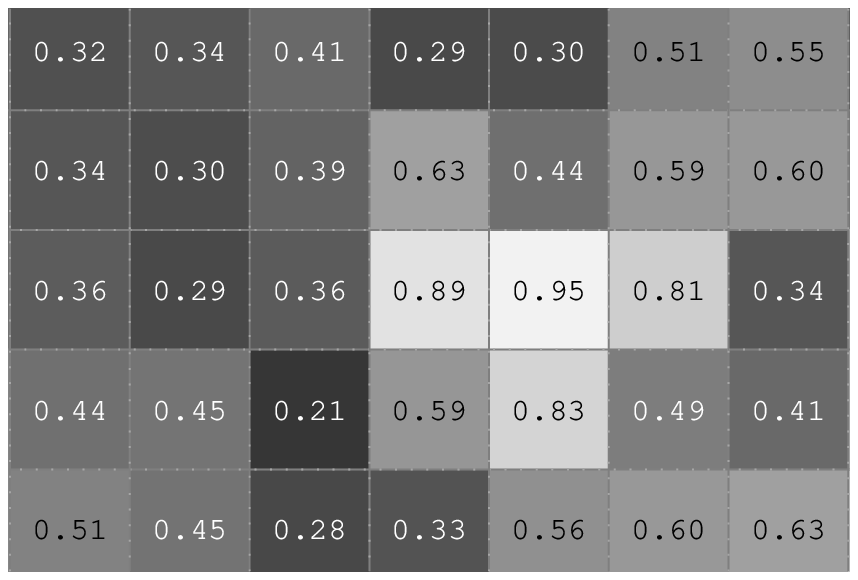}\label{fig:A2}}
 \subfigure[$X\longrightarrow A3(\circ)$]{
  \includegraphics[width=57 mm]{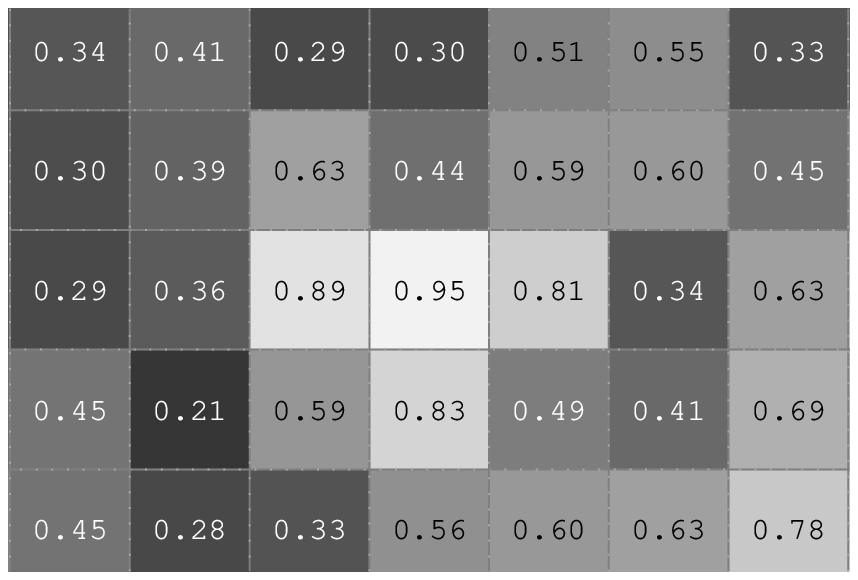}\label{fig:A3}}
 \caption{Neighbours of $A(\circ)$ in~Fig.~\ref{fig:tiger1}}\label{fig:nbdA}
\end{center}
\end{figure}

\subsection{Digital groupoid set pattern generators}
A \emph{pattern generator} is derived some form of regular structure~\cite{Grenander1993}.  In geometry, a regular polygon is a polygon with $n$ sides that have the same length and are symmetrically place around a common center.  A geometric pattern contains a repetition of a regular polygon.  An element in a groupoid $A(\circ)$ is regular, provided $a\in aAa$, {\em i.e.},  $a\circ y\circ a = a$ for some $y\in A$~\cite{Clifford1964}.  A groupoid is regular, provided every element is regular.  
\begin{example}
From example~\ref{ex:neighbourly}, let $\Phi(x), \Phi(y)\in \mathcal{Q}(A)$ be neighbourly elements of groupoid A.  Then
\[
\Phi(x)\circ \Phi(y)\circ \Phi(x) = \Phi(y)\circ \Phi(x) = \Phi(x),\ \mbox{since}, \Phi(x) = \Phi(x).
\]
Hence, $\Phi(x)$ is a regular element.
\end{example}
An algebraic pattern contains a repetition of regular structures such as a regular element or a regular groupoid. Again, for example, an algebraic pattern contains a repetition of neighbourly elements from groupoids such as the one in Example~\ref{ex:neighbourly}.  In this work, the focus is proximal algebraic patterns.  A \emph{proximal algebraic pattern} in a descriptive proximity space $(X,\delta_{\Phi})$ (denoted by $\mathfrak{P}(A)$) is derived from a repetition of structures such as groupoids $A$ that contain a mixture of neighbourly (regular) and non-neighbourly (non-regular) elements, {\em i.e.},
\begin{align*}
\mathfrak{P}_{\Phi}(A(\circ)) = &\{B(\bullet): B\in 2^{\mathcal{Q}(X)}\ \mbox{and}\ \Phi(a) = \Phi(b),\\
                         &\mbox{for some},\ \Phi(a)\in \mathcal{Q}(A), \Phi(b)\in \mathcal{Q}(B)\}.
\end{align*}
  
\begin{example}\label{ex:neighbourlyPatterns}
From example~\ref{ex:neighbourly}, groupoid $(\mathcal{Q}(A),\circ)$.  From the digital image in Fig.~\ref{fig:tiger1}, we can find additional groupoids containing neighbourly elements.  Two such groupoids are given in Fig.~\ref{fig:nbdA}.  That is, the proximal groupoid $A2(\circ)$ given in Fig.~\ref{fig:A2} is derived from the digital image in Fig.~\ref{fig:tiger1} and has a number of neighbourly elements such as those pixels with intensities 0.36, 0.41, 0.44, 0.59.  And $A2(\circ)$ and groupoid $A(\circ)$ (in Fig.~\ref{fig:tiger1}) are neighbourly ({\em e.g.}, both $A$ and $A2$ have pixels with intensity 0.83).   Similarly, the proximal groupoid $A3(\circ)$ given in Fig.~\ref{fig:A3} is also derived from the digital image in Fig.~\ref{fig:tiger1} and has a number of neighbourly elements.  Also notice that groupoid $A3$ also a pixel with intensity 0.83.  Then, from Theorem~\ref{thm:nbdPattern} and the definition of a proximal algebraic pattern, we obtain
\[
\mathfrak{P}_{\Phi}(A) = \left\{A,A2,A3\right\}.
\]
Similarly, proximal groupoids $B2,B3$ from Fig.~\ref{fig:nbdB} are derived from the digital image in Fig.~\ref{fig:tiger2}.  In addition, $B2,B3$ and the proximal groupoid $B$ in Fig.~\ref{fig:tiger2} are neighbourly.  Hence, we obtain a second proximal algebraic pattern, namely,
\[
\mathfrak{P}_{\Phi}(B) = \left\{B,B2,B3\right\}.
\] 
Patterns $\mathfrak{P}_{\Phi}(A),\mathfrak{P}_{\Phi}(B)$ are examples of neighbourly patterns. \qquad \textcolor{blue}{$\blacksquare$}
\end{example}

\begin{theorem}~\label{thm:generator}
In a descriptive proximity space, a proximal groupoid containing regular elements generates 
a proximal groupoid pattern.
\end{theorem}
\begin{proof}
Assume $A$ in a descriptive proximity space $X$ is the only proximal groupoid containing neighbourly elements.  Then $\mathfrak{P}(A) =\left\{A\right\}$ is a proximal groupoid pattern containing only the groupoid $A$.  Let $A,B$ be a pair of neighbourly proximal groupoids in $X$.  Then,  
from Theorem~\ref{thm:nbdPattern}, $A\ \delta{\Phi}\ B$ and either $A$ or $B$ is a generator of a proximal groupoid pattern.
\end{proof}

\begin{theorem}~\label{thm:groupoidPattern}
In a descriptive proximity space, a proximal groupoid pattern is a collection of neighbourly groupoids.
\end{theorem}
\begin{proof}
Immediate from Theorem~\ref{thm:generator} and the definition of a proximal algebraic pattern.
\end{proof}

The likelihood of finding proximal algebraic patterns in nature is high, since natural patterns seldom, if ever, contains regular structures such as regular polygons or regular algebraic structures.  Proximal algebraic patterns are commonly found in digital images.  Hence, proximal algebraic structures are useful in classifying digital images.

\begin{figure}[!ht]
\begin{center}
 \subfigure[$X\longrightarrow B2(\circ)$]{
  \includegraphics[width=57 mm]{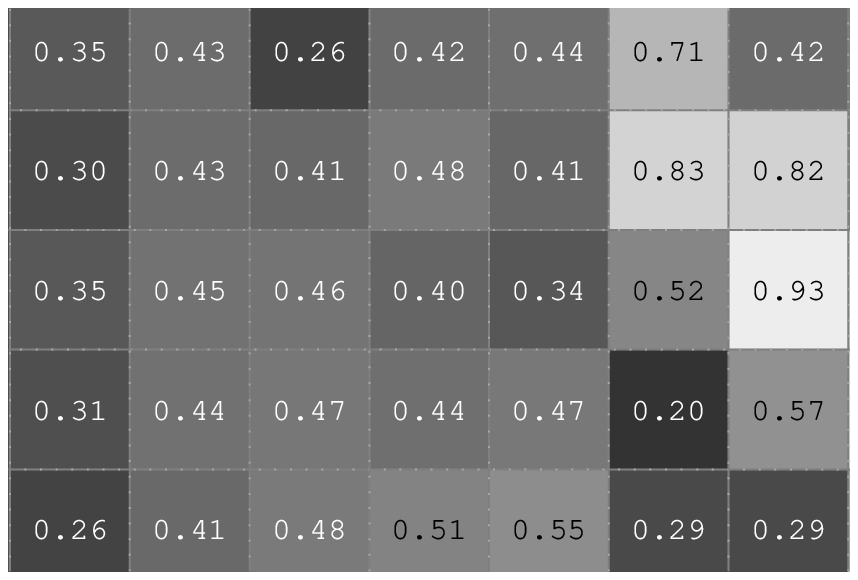}\label{fig:B2}}
 \subfigure[$X\longrightarrow B3(\circ)$]{
  \includegraphics[width=57 mm]{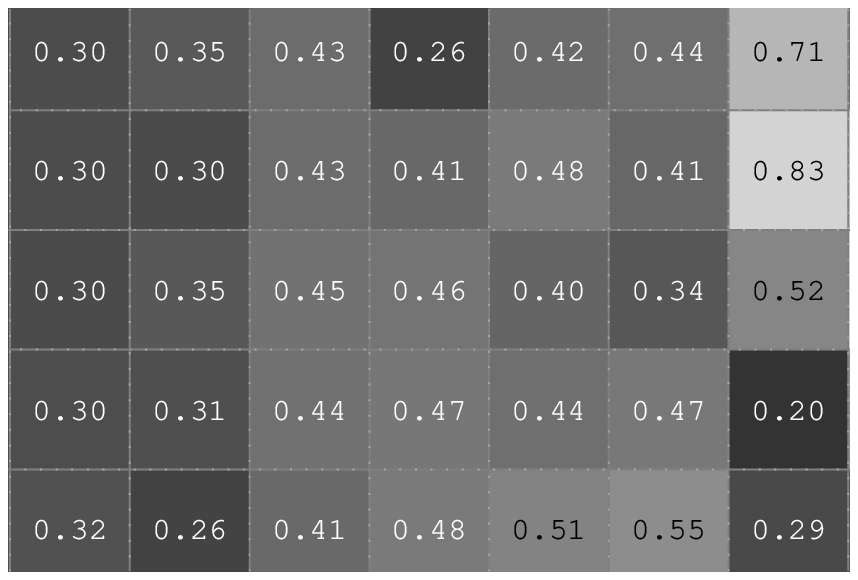}\label{fig:B3}}
 \caption{Neighbours of $B(\circ)$ in~Fig.~\ref{fig:tiger2}}\label{fig:nbdB}
\end{center}
\end{figure}

\subsection{Neighbourly patterns in digital images}
Let $X,Y$ be descriptive proximity spaces and let $A\subset X, B\subset Y$ be proximal groupoids that generate the proximal algebraic patterns $\mathfrak{P}(A),\mathfrak{P}(B)$ in $X,Y$, respectively.  A pair of patterns $\mathfrak{P}(A),\mathfrak{P}(B)$ are neighbourly, provided $A\ \delta{\Phi}\ B$, {\em i.e.}, the description $\Phi(a)$ matches $\Phi(a)$ for some $a\in A, b\in B$.  In other words, \emph{neighbourly proximal algebraic patterns} contain neighbourly sets.

\begin{example}
In Example~\ref{ex:neighbourlyPatterns}, patterns $\mathfrak{P}_{\Phi}(A),\mathfrak{P}_{\Phi}(B)$ are neighbourly.  To see this, consider, for instance, groupoid $A$ in Fig.~\ref{fig:tiger1} has a pixel with intensity 0.83 (labelled with a $\star$ in Fig.~\ref{fig:tiger1}).  Similarly, groupoid $B$ in Fig.~\ref{fig:tiger2} has a pixel with intensity 0.83 (also labelled with a $\star$).  Consequently, $A$ and $B$ are neighbourly.  Hence, groupoids $A, B$ generate neighbourly patterns. \qquad \textcolor{blue}{$\blacksquare$}
\end{example} 

In the search for neighbourly patterns in digital images $X,Y$, a proximal algebraic pattern $\mathfrak{P}_Y(B)$ in $Y$ is salient (important) relative to a proximal algebraic pattern $\mathfrak{P}_X(A)$ in $X$, provided the number of neighbourly elements $a\in A, b\in B$ is high enough.  A digital image $Y$ belongs to the class of images represented by $X$, if $Y$ contains a salient pattern that is neighbourly with a pattern in $X$.  In that case, images $X,Y$ are well-classified.

\begin{theorem}
Let $(X,\delta_{\Phi}),(Y,\delta_{\Phi})$ be descriptive proximity spaces and let $A\subset X, B\subset Y$ generate neighbourly patterns $\mathfrak{P}_{\Phi}(\mbox{cl}A),\mathfrak{P}_{\Phi}(\mbox{cl}B)$.  If every $b\in B$ is neighbourly with some $a\in A$, then $\mathfrak{P}_{\Phi}(\mbox{cl}B)$ is a salient proximal algebraic pattern. 
\end{theorem}
\begin{proof}
Immediate from the definition of a salient proximal algebraic pattern.
\end{proof}

\bibliographystyle{amsplain}
\bibliography{mybibfile}

\end{document}